\documentclass[review]{elsarticle}

\usepackage{amssymb}

\usepackage{lineno}
\usepackage{graphicx}
\usepackage{epstopdf}
\usepackage{subfigure}
\usepackage{amsmath}
\usepackage{multirow}
\usepackage[bookmarks=true]{hyperref}
\usepackage{algorithm}
\usepackage{algorithmicx, algpseudocode}
\usepackage{picins}
\newtheorem{theorem}{Theorem}
\newproof{proof}{Proof}
\newproof{pot}{Proof of Theorem \ref{thm2}}

\def \myalg {SemFS}
\def \myalgv {SemFS/c}
\newtheorem{proposition}[theorem]{Proposition}

\makeatletter %
\newcommand\mysubsection{\@startsection{paragraph}{4}{\z@}{3\p@ \@plus \p@}{-5\p@}{\normalsize\bfseries}}

\journal{arXiv}
\begin{document}

\begin{frontmatter}



\title{Zero-shot  Feature  Selection  via  Transferring Supervised Knowledge}


%

\author[mymainaddress]{Zheng Wang}
\ead{wangzheng@ustb.edu.cn}
\author[mysecondaryaddress]{Qiao Wang}
\author[mymainaddress]{Tingzhang Zhao}

\author[mysecondaryaddress]{Xiaojun Ye}

\address[mymainaddress]{Department of Computer Science, University of Science and Technology Beijing, Beijing 100083, P.R. China}
\address[mysecondaryaddress]{School of Software, Tsinghua University, Beijing 100084, P.R. China}

\begin{abstract}
Feature selection, an effective technique for dimensionality reduction, plays an important role in many machine learning systems. Supervised knowledge can significantly improve the performance.
However, faced with the rapid growth of newly emerging concepts, existing supervised methods might easily suffer from the scarcity and validity of labeled data for training.
In this paper, the authors study the problem of zero-shot feature selection (i.e., building a feature selection model that generalizes well to “unseen” concepts with limited training data of “seen” concepts).
Specifically, they adopt class-semantic descriptions (i.e., attributes) as supervision for feature selection, so as to utilize the supervised knowledge transferred from the seen concepts.
For more reliable discriminative features, they further propose the center-characteristic loss which encourages the selected features to capture the central characteristics of seen concepts.
Extensive experiments conducted on various real-world datasets demonstrate the effectiveness of the method.
\end{abstract} 


\end{frontmatter}


\begin{abstract}
Feature selection, an effective technique for dimensionality reduction, plays an important role in many machine learning systems. Supervised knowledge can significantly improve the performance.
However, faced with the rapid growth of newly emerging concepts, existing supervised methods might easily suffer from the scarcity and validity of labeled data for training.
In this paper, the authors study the problem of zero-shot feature selection (i.e., building a feature selection model that generalizes well to “unseen” concepts with limited training data of “seen” concepts).
Specifically, they adopt class-semantic descriptions (i.e., attributes) as supervision for feature selection, so as to utilize the supervised knowledge transferred from the seen concepts.
For more reliable discriminative features, they further propose the center-characteristic loss which encourages the selected features to capture the central characteristics of seen concepts.
Extensive experiments conducted on various real-world datasets demonstrate the effectiveness of the method.
\end{abstract} 
\section{Introduction}
The problem of feature selection~\cite{guyon2003introduction}~\cite{lin2018sharing} has been widely investigated due to its importance for pattern recognition and image processing systems.
This problem can be formulated as follows: identify an optimal feature subset which provides the best tradeoff between its size and relevance for a given task.
The identified features not only provide an effective solution for the task, but also provide a dimensionally-reduced view of the underlying data~\cite{bicciato2003pca}.

\begin{figure}[!t]
\centering
    \includegraphics[width=0.6\textwidth]{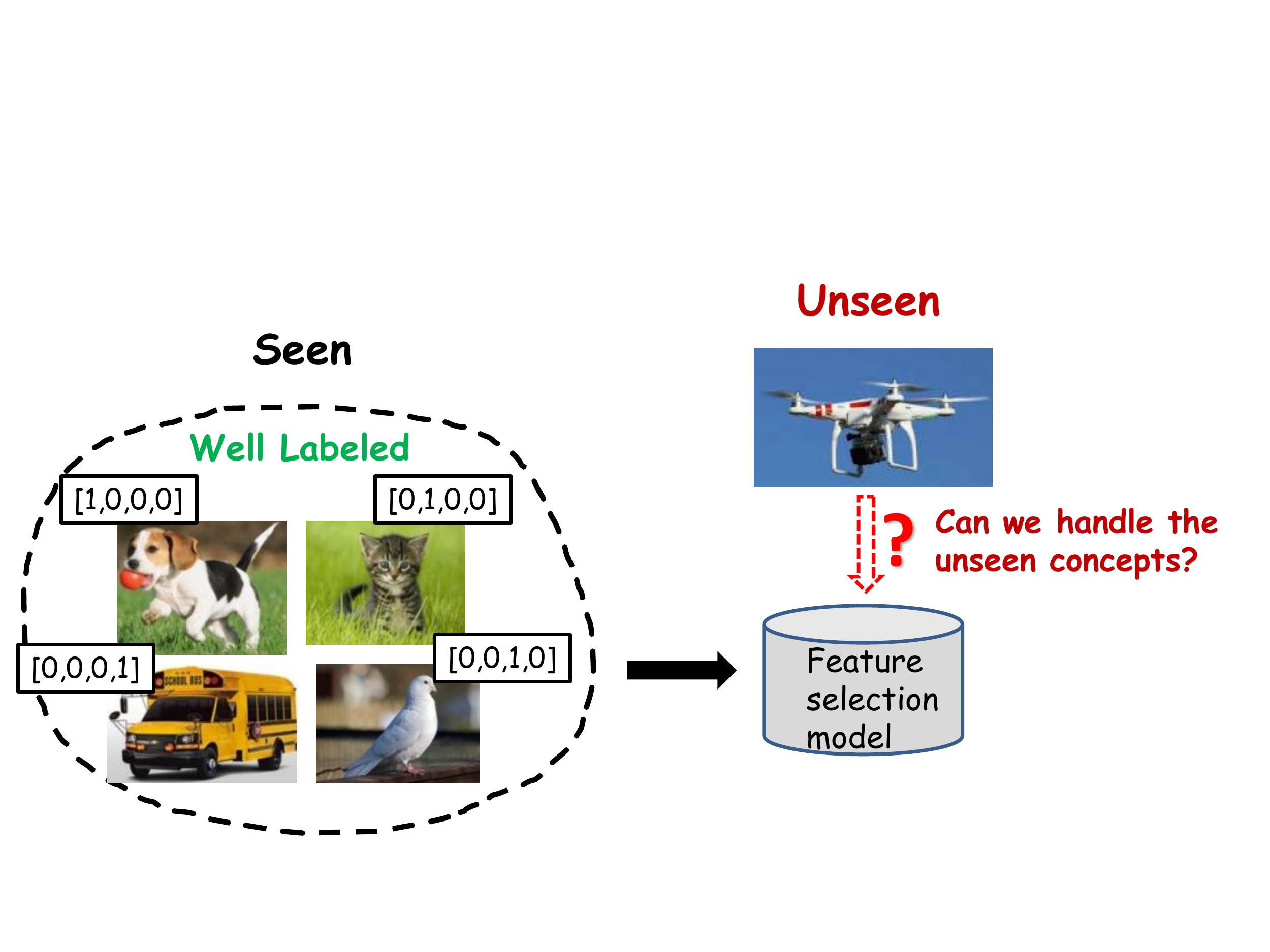}
\caption{
The Zero-Shot Feature Selection Problem.
}
\label{fig_problem}
\end{figure}
Supervised knowledge (e.g., labels or pair-wise relationships) associated to data is capable of significantly improving the performance of feature selection methods~\cite{chandrashekar2014survey}.
However, it should be noted that existing supervised feature selection methods are facing an enormous challenge --- the generation of reliable supervised knowledge cannot catch up with the rapid growth of newly-emerging concepts and multimedia data.
In practice, it is costly to annotate sufficient training data for the new concepts timely, and meanwhile, impractical to retrain the feature selection model whenever a new concept emerges.
As illustrated in Fig.~\ref{fig_problem}, traditional methods perform well on the seen concepts which have correct guidance, but they may easily fail on the unseen concepts which have never been observed, like the newly invented product ``quadrotor''.
Therefore, the problem of \emph{Zero-Shot Feature Selection (ZSFS)}, i.e., building a feature selection model that generalizes well to unseen concepts with limited training data of seen concepts, deserves great attention.
However, few studies have considered this problem.
\begin{figure}[!t]
\centering
    \includegraphics[width=0.8\textwidth]{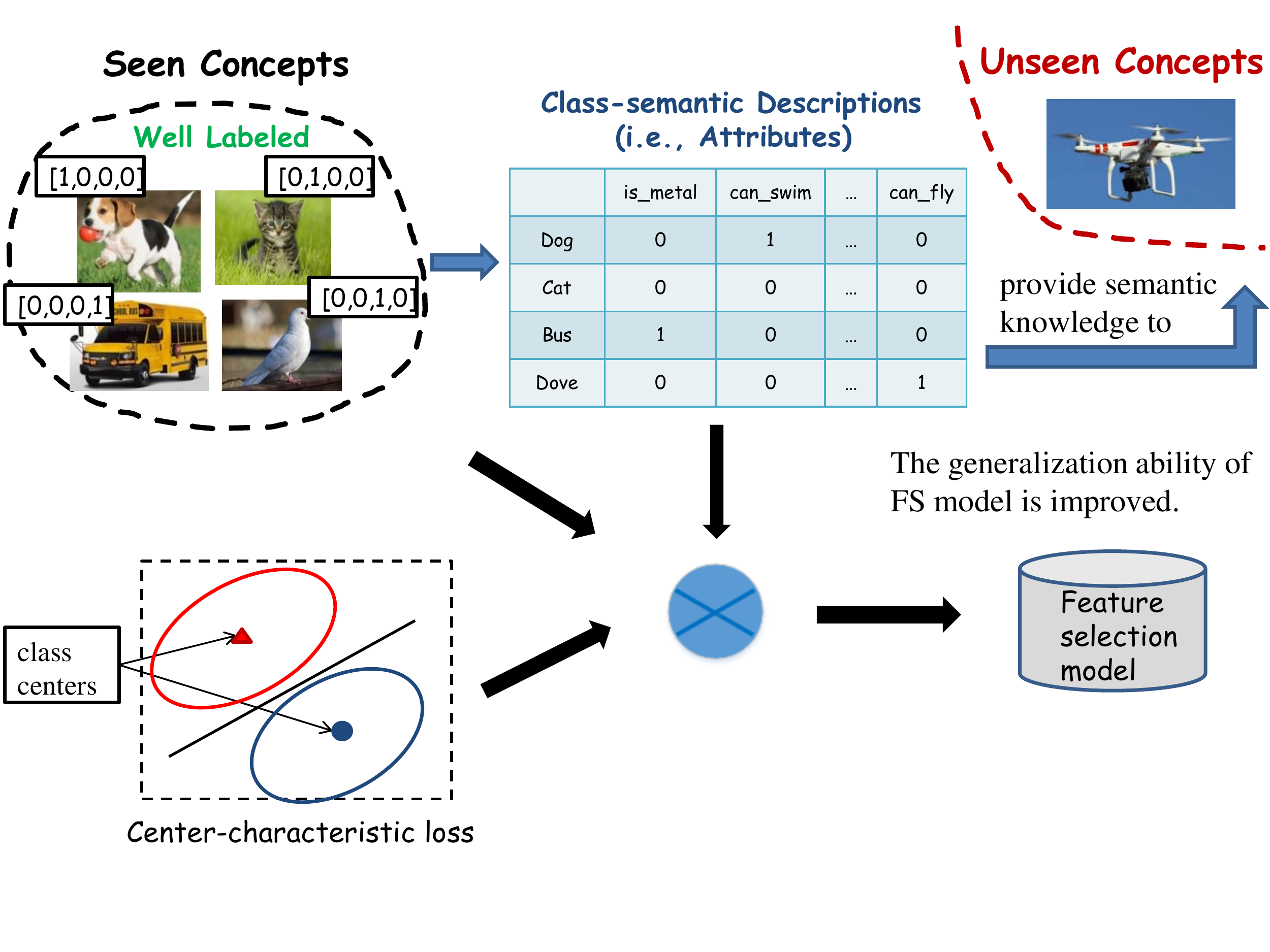}
\caption{Overview of the proposed method.}
\label{fig_method}
\end{figure}

The major challenge in the ZSFS problem is how to deduce the knowledge of unseen concepts from seen concepts.
In fact, the primary reason why existing studies fail to handle unseen concepts is that they only consider the discrimination among seen concepts (like the 0/1-form class labels illustrated in Fig.~\ref{fig_problem}), such that little knowledge could be deduced for unseen concepts.
To address this, as illustrated in Fig.~\ref{fig_method}, we adopt the class-semantic descriptions (i.e., attributes) as supervision for feature selection.
This idea is inspired by the recent development of Zero-Shot Learning (ZSL)~\cite{farhadi2009describing}~\cite{akata2016multi}~\cite{ji2019manifold} which has demonstrated that the capacity of inferring attributes allows us to describe, compare, or even categorize unseen objects.

An attendant problem is how to identify reliable discriminative features with attributes which might be inaccurate and noisy~\cite{jayaraman2014zero}.
To alleviate this, we further propose a novel loss function (named \emph{center-characteristic loss}) which encourages the selected features to capture the central characteristics of seen concepts.
Theoretically, this loss function is a variant of the \emph{center loss}~\cite{wen2016discriminative} which has shown its effectiveness to learn discriminative and generalized features for categorizing unseen objects.

We evaluate the performance of the proposed method on several real-world datasets, including SUN, aPY and CIFAR10.
One point should be noted is that the attributes of CIFAR10 are automatically generated from a public Wikipedia text-corpus~\cite{shaoul2010westbury} by a well-known NLP tool~\cite{huang2012improving}.
The experimental evidence shows that no matter with manually or automatically generated attributes, our method generalizes well to unseen concepts.
We summarize our main contributions as follows:
\begin{itemize}
  \item
  We study the problem of ZSFS, i.e., building a feature selection model which generalizes well to unseen concepts with limited training data of seen concepts.
  To our best knowledge, little work has addressed this problem.
  \item We propose an efficient strategy to reuse the supervised knowledge of seen concepts.
    Concretely, feature selection is guided by the seen concept attributes which provide discriminative information about unseen concepts.
  \item To select more reliable discriminative features, we further encourage the selected features to capture the central characteristics of seen concepts.
  We formulate this by a novel loss function named center-characteristic loss.
  \item
  We conduct extensive experiments on three benchmark datasets to demonstrate the effectiveness of our method.
\end{itemize}

The rest of this paper is organized as follows. In Section~\ref{sect_related}, we briefly review some related work in feature selection and zero-shot learning.
In Section~\ref{section_problem_statement}, we formally define the problem studied in this paper.
In Section~\ref{sect_method}, we elaborate our approach with details, together with the optimization method.
In Section~\ref{sect_discuss}, more analysis of the method is provided.
Extensive experiments on several different datasets will be reported in Section~\ref{sect_expriment}, followed by the conclusion in Section~\ref{section_conclusion}.

\section{Related Work}\label{sect_related}
\subsection{Feature Selection}
Feature selection, which selects a subset of the original features according to some criteria, plays an important role in pattern recognition and machine learning systems.
In terms of the label availability, existing feature selection methods can be roughly classified into three groups: unsupervised, supervised and semi-supervised feature selection methods.
In the unlabeled case, unsupervised methods select features which best keep the intrinsic structure of data according to various criteria, such as data variance~\cite{dash2000feature}, data similarity~\cite{nie2008trace,chu2018bidirectional} and data separability~\cite{cai2010unsupervised}.
To take advantage of labeled data, supervised methods~\cite{lee1988thirteen,cai2013exact,zhao2018trace} evaluate features by their relevance to the given class labels.
Extended from unsupervised and supervised methods, semi-supervised methods~\cite{kong2010semi,xu2010discriminative,zeng2016semi} utilize both labeled and unlabeled data to mine the feature relevance.

It should be noted that traditional feature selection methods do not consider the generalization to ``unseen'' concepts, limiting in the ``seen'' area where every concept should at least provide one (unlabeled or labeled) instance.
In light of the knowledge explosion, with the rapid growth of newly-emerging concepts and multimedia data, building a feature selection model that generalizes well to unseen concepts has practical importance.

\subsection{Attribute-based Zero Shot Learning}
The task of Zero Shot Learning (ZSL)~\cite{lampert2009learning,farhadi2009describing,akata2016multi} is to recognize unseen objects without any label information.
To achieve this goal, it leverages an intermediate semantic level (i.e., attribute layer) which is shared in both seen and unseen concepts.
Take the seen concepts in Fig.~\ref{fig_method} as an example.
We can define some attributes like ``can swim'', ``can fly'' and ``is metal''.
Then, we can train attribute recognizers using images and attribute information from seen concepts.
After that, given an image belonging to unseen concepts, these attribute recognizers can infer the attributes of this image.
Finally, the recognition result is obtained by comparing the test image's attributes with each unseen concept's attributes.

Although various ZSL methods have been proposed recently, they are mainly limited to classification or prediction applications.
To our best knowledge, this is the first study to consider the zero-shot setting in the feature selection problem.



\section{Problem Definition and Notations}\label{section_problem_statement}
Suppose there is a set of instances $X{=}\{x_{1},...,x_{n}\}'\in \mathbb{R}^{n \times d}$ belonging to a seen concept set $\mathcal{C}$, where $x_{i}\in \mathbb{R}^{d}$ is the feature vector and $n$ is the instance number.
We denote $Y{=}\{y_{1},...,y_{n}\}' \in \{0,1\}^{n \times c}$ as the binary label indicator matrix, where $y_{i} \in \{0,1\}^{c}$ is the label vector of instance $x_{i}$ and $c$ is the number of seen concepts in $\mathcal{C}$.


Different from the traditional supervised feature selection setting where training and testing instances all belong to the same seen concept $\mathcal{C}$, the problem of \emph{Zero-Shot Feature Selection (ZSFS)} considers a more challenging setting where testing instances belong to a related but ``unseen'' concept set $\mathcal{C}^{u}$.
In other words, training and testing instances share no common concepts: $\mathcal{C}\cap \mathcal{C}^{u} {=} \varnothing$.
Using only the training instances $X$ belonging to the seen concepts in $\mathcal{C}$, our goal is to learn a feature selection model that generalizes well to the unseen concepts in $\mathcal{C}^{u}$.

%


\section{The Proposed Method}\label{sect_method}

In this section, we provide a detailed description of the proposed method.
Firstly, we discuss the feasibility of deducing knowledge of unseen concepts from seen concepts.
Then, we describe the formulation of our method.
Finally, we provide an effective solution to address the involved optimization issue.

\subsection{Deduce Knowledge for Unseen Concepts}
The success of ZSL demonstrates that the capacity of inferring attributes allows us to describe, compare, or even categorize unseen objects.
Taking the seen concepts in Fig.~\ref{fig_method} as an example, with the training images of these concepts, we could learn a mapping function between image features and some attributes (e.g., ``has stripes'', ``can fly'' and ``is metal'').
As such, we can infer the attributes of the objects belonging to unseen concepts such as ``koala'' and ``quadrotor'', so as to classify them without any training examples.


In light of this, as illustrated in Fig.~\ref{fig_method}, we propose to replace the original class labels with class-attribute descriptions (i.e., attributes) to guide feature selection.
Unlike the original class labels (e.g., the 0/1-form label vectors in Fig.~\ref{fig_problem}) which only reflect the discrimination among seen concepts, attributes provide additional semantic information about unseen concepts, making the feature selection models trained with seen concepts generalize well to unseen concepts.
In other words, by introducing attributes, we can deduce the knowledge of unseen concepts from seen concepts.
In the subsequent part, we denote this kind of semantic knowledge as $Y_{s} = [y^{s}_{1}, ..., y^{s}_{n}]' \in \mathbb{R}^{n \times m}$, where $y^{s}_{i} \in \mathbb{R}^{m}$ denotes the attributes of the seen concept which instance $x_{i}$ belongs to, and $m$ is the attribute number.
These attributes can either be provided manually or generated automatically from online textual documents~\cite{rohrbach2010helps,qiao2016less}, such as Wikipedia articles.

\subsection{Zero-Shot Feature Selection}
\subsubsection{Feature selection with attributes}
We use $s {=} \{0, 1\}^{d}$ as the indicator vector, where $s_{i} $=$ 1$ if the $i$-th feature is selected and $s_{i} $=$ 0$ otherwise.
Denoting $\mathrm{diag}(s)$ as the matrix with the main diagonal being $s$, the original data can be represented as $X\mathrm{diag}(s)$ with the selected features.
From a generative point of view, we assume that the selected features should have the ability to generate the given attributes.
For simplicity and efficiency, we use a linear generating function $W {\in} \mathbb{R}^{d \times m}$, and adopt the squared loss to measure the error.
Therefore, the optimal $s$ can be obtained by solving the following minimization problem:

\begin{equation}
\label{eq_semantic_selection}
\begin{aligned}
\underset{W, s}{\text{min }} & \mathcal{J}_{1} = \left \| Y_{s}-X\mathrm{diag}(s)W \right \|^{2}_{F} + \gamma \left \| W \right \|^{2}_{F} \\
\mathrm{s.t.}   & \hspace{0.5em} s \in \{0,1\}^{d},\ s^{T}\textbf{1}_{d}=k
\end{aligned}
\end{equation}
where $k$ is the number of features to select, $\gamma$ is the regularization parameter to avoid overfitting, and $\textbf{1}_{d}$ is a column vector with all its elements being 1.
The most important part of Eq.~\ref{eq_semantic_selection} is the attributes $Y_{s}$ through which the power of categorizing unseen objects is captured by the selected features.


\subsubsection{The center-characteristic loss}

\begin{figure}[!t]
\centering
    \includegraphics[width=0.8\textwidth]{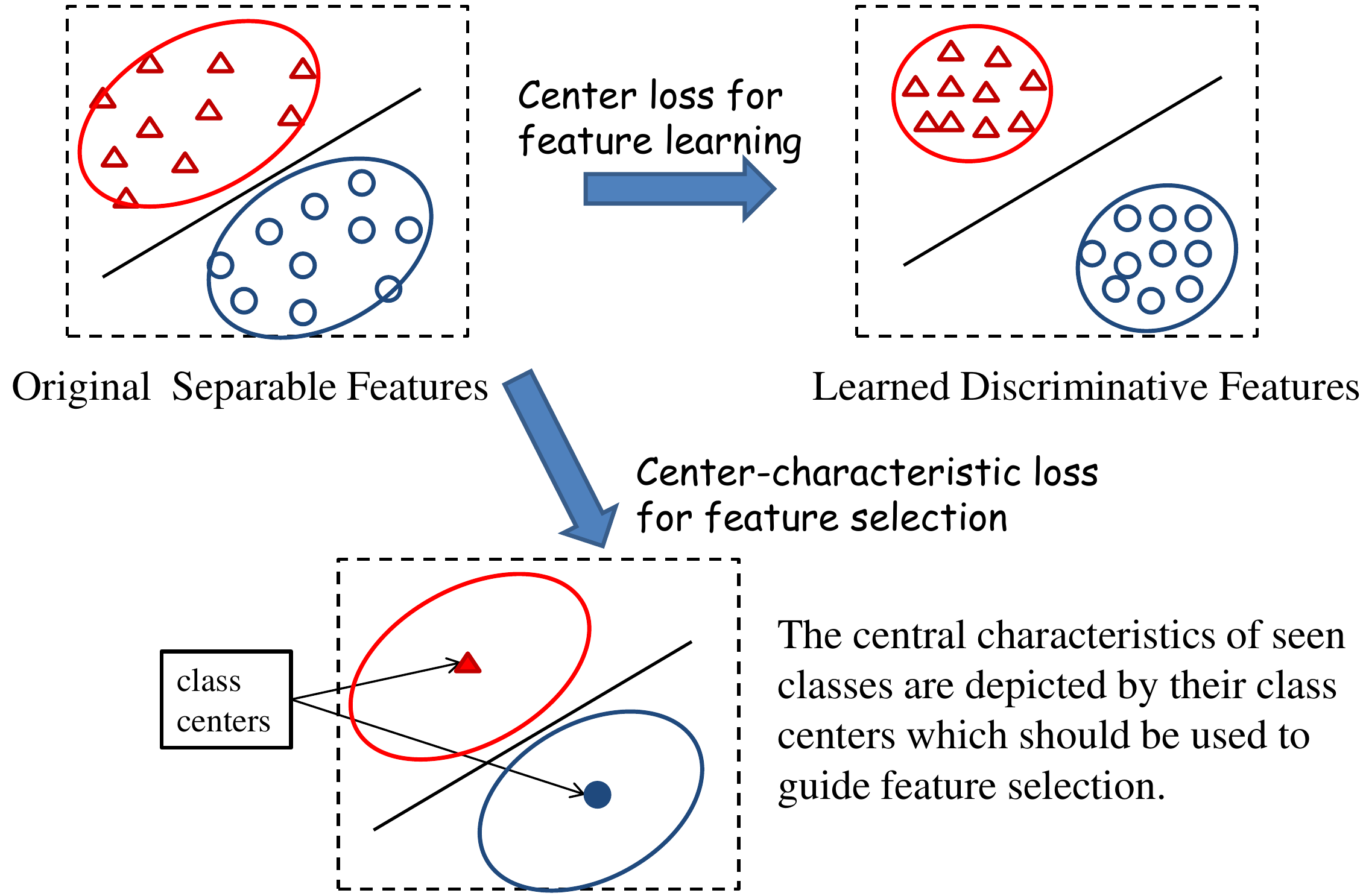}
\caption{Center loss vs. Center-characteristic loss.}
\label{fig_center-loss}
\end{figure}

Although introducing attributes seems to be an effective solution for the ZSFS problem, it still has some limitations for discriminative feature selection.
In particular, attributes cannot naturally describe the uncertainty about a class whose appearance may vary significantly~\cite{ren2016joint}.
For instance, ``televisions'' may have different colors and look quite differently from different angles.
Moreover, attributes, which can be seen as real-valued multi-labels, tend to be more noisy than the common single-class labels~\cite{jayaraman2014zero}.
For example, it is hard to determine a proper score for the attribute ``running fast'' for different animals and man-made vehicles.

Recent work~\cite{wen2016discriminative} has pointed out that discriminative features are more generalized for identifying unseen objects than separable features.
Specifically, to learn discriminative features, the work in \cite{wen2016discriminative} introduces \emph{center loss} to minimize the distances between the learned features and their corresponding class centers (as illustrated in the upper portion of Fig.~\ref{fig_center-loss}).
Accordingly, for the feature selection task, this loss minimizes the distances between the selected features and their corresponding class centers, which yields the following minimization problem:
\begin{equation}
\label{eq_center-loss}
\begin{aligned}
\underset{s}{\text{min }} & \mathcal{J}_{2} = \left \| X\mathrm{diag}(s) -\bar{X}\mathrm{diag}(s) \right \|^{2}_{F} \\
\mathrm{s.t.}   & \hspace{0.5em} s \in \{0,1\}^{d},\ s^{T}\textbf{1}_{d}=k
\end{aligned}
\end{equation}
where $\bar{X} = \{\bar{x}_{1}, ..., \bar{x}_{n} \}'\in \mathbb{R}^{n \times d}$, and $\bar{x}_{i}$ denotes the feature vector of the $y_{i}$-th class center.

However, this constraint may be too strict for a feature selection task, since all features are pre-calculated.
On the other hand, as our ultimate goal is to find the relevance between features and attributes, it is reasonable to incorporate the attributes into this loss function.
Therefore, we modify this constraint by multiplying the terms in the square loss with the generating function $W$ and approximate $X\mathrm{diag}(s)W$ with $Y_{s}$ according to Eq.~\ref{eq_semantic_selection}.
After this modification, Eq.~\ref{eq_center-loss} becomes:
\begin{equation}
\label{eq_self-center-loss}
\begin{aligned}
\underset{W,s}{\text{min }} & \mathcal{J}_{2} = \left \| Y_{s}-\bar{X}\mathrm{diag}(s)W \right \|^{2}_{F}  \\
\mathrm{s.t.}   & \hspace{0.5em} s \in \{0,1\}^{d},\ s^{T}\textbf{1}_{d}=k
\end{aligned}
\end{equation}

We call this loss function \emph{center-characteristic loss}, as it encourages the selected features to capture the central characteristics of seen concepts (illustrated in the lower portion of Fig.~\ref{fig_center-loss}).


\subsubsection{Zero-Shot feature selection joint learning model}
By combining the objective functions in Eqs.~\ref{eq_semantic_selection} and~\ref{eq_self-center-loss}, the proposed method is to solve the following optimization problem:

\begin{equation}
\small
\label{eq_final_cost01}
\begin{aligned}
\underset{W, s}{\text{min }} & \mathcal{J} = \left \| Y_{s}{-}X\mathrm{diag}(s)W \right \|^{2}_{F} {+} \alpha \left \| Y_{s}{-}\bar{X}\mathrm{diag}(s)W \right \|^{2}_{F} \\
& \hspace{2em} + \gamma \left \| W \right \|^{2}_{F} \\
\mathrm{s.t.}   & \hspace{0.5em} s \in \{0,1\}^{d},\ s^{T}\textbf{1}_{d}=k
\end{aligned}
\end{equation}
where $\alpha$ is a balancing parameter. We call the proposed method as \emph{Semantic based Feature Selection (\myalg)}, since we exploit the semantic knowledge (i.e., attributes) of seen concepts for the ZSFS problem.


Note that our method can be extended to handle both seen and unseen concepts, if we (just like the traditional feature selection methods) incorporate the original class labels into the final objective function.
We leave this extension as future work.


\subsection{Optimization}
The optimization problem in Eq.~\ref{eq_final_cost01} is a `0/1' integer programming problem, which might be hard to solve by conventional optimization tools.
Therefore, in this subsection, we give an efficient solution for this problem.
First of all, to make the optimization tractable, we relax this `0/1' constraint by allowing $s$ to take real non-negative values.
This relaxation yields the following optimization problem:
\begin{equation}
\small
\label{eq_final_cost_relax}
\begin{aligned}
\underset{W, s}{\text{min }} & \mathcal{J} = \left \| Y_{s}{-}X\mathrm{diag}(s)W \right \|^{2}_{F} {+} \alpha \left \| Y_{s}{-}\bar{X}\mathrm{diag}(s)W \right \|^{2}_{F} \\
& \hspace{2em} + \gamma \left \| W \right \|^{2}_{F} \\
\mathrm{s.t.}   & \hspace{0.5em} s \geq 0
\end{aligned}
\end{equation}

As such, the objective function is differentiable w.r.t. the real-valued $s$.
Intuitively, the value of $s_{i}$ can be interpreted as the importance score of $i$-th feature.
Important features would have higher scores, while the scores of unimportant features tend to shrink towards 0.
Therefore, after obtaining the solution of Eq.~\ref{eq_final_cost_relax}, we can rank $s$ to identify important features.


The optimization problem in Eq.~\ref{eq_final_cost_relax} can be solved iteratively.
Specifically, this consists of the following two steps.

\mysubsection{Fix $s$ and update $W$}
When $s$ is fixed, Eq.~\ref{eq_final_cost_relax} is convex w.r.t $W$.
Therefore, after removing the terms that are irrelevant to $W$, we can reformulate the objective function in Eq.~\ref{eq_final_cost_relax} as the following unconstrained optimization problem:


\begin{equation}
\small
\label{eq_cost_relax_W}
\begin{aligned}
\underset{W}{\text{min }} & \mathcal{J}_{W} {=} \left \| Y_{s}{-}X\mathrm{diag}(s)W \right \|^{2}_{F} {+} \alpha \left \| Y_{s}{-}\bar{X}\mathrm{diag}(s)W \right \|^{2}_{F} \\
& \hspace{2em} + \gamma \left \| W \right \|^{2}_{F}
\end{aligned}
\end{equation}
The derivative of $\mathcal{J}_{W}$ w.r.t. $W$ is:
\begin{equation}
\label{eq_derive_W}
\begin{aligned}
&\frac{\partial \mathcal{J}_{W}}{\partial W} = -2(X\mathrm{diag}(s))'Y_{s} + 2(X\mathrm{diag}(s))'(X\mathrm{diag}(s))W\\
 &-2\alpha(\bar{X}\mathrm{diag}(s))'Y_{s} {+} 2\alpha(\bar{X}\mathrm{diag}(s))'(\bar{X}\mathrm{diag}(s))W {+} 2\gamma W\\
\end{aligned}
\end{equation}

\noindent By setting this derivative to zero, we obtain the closed form solution for $W$:
\begin{equation}
\label{eq_update_W}
\begin{aligned}
W {=}& [(X\mathrm{diag}(s))'(X\mathrm{diag}(s)) {+} \alpha (\bar{X}\mathrm{diag}(s))'(\bar{X}\mathrm{diag}(s)) \\
& {+} \gamma I_{d}]^{-1}[(X\mathrm{diag}(s))' {+} \alpha (\bar{X}\mathrm{diag}(s))']Y_{s}
\end{aligned}
\end{equation}

\mysubsection{Fix $W$ and update $s$}
When $W$ is fixed, the objective function in Eq.~\ref{eq_final_cost_relax} can be rewritten as follows:
\begin{equation}
\small
\label{eq_cost_relaxed_s}
\begin{aligned}
\underset{s}{\text{min }} & \mathcal{J}_{s} {=} \left \| Y_{s}{-}X\mathrm{diag}(s)W \right \|^{2}_{F} {+} \alpha \left \| Y_{s}{-}\bar{X}\mathrm{diag}(s)W \right \|^{2}_{F}\\
\mathrm{s.t.}   & \hspace{0.5em} s \geq 0
\end{aligned}
\end{equation}


\noindent The derivative of $\mathcal{J}_{s}$ w.r.t. $s$ is:
\begin{equation}
\label{eq_update_s}
\begin{aligned}
\frac{\partial \mathcal{J}_{s}}{\partial s} &= [- 2 X'Y_{s}W' + 2 X'X\mathrm{diag}(s)WW' \\
& -2\alpha \bar{X}'Y_{s}W' + 2\alpha \bar{X}'\bar{X}\mathrm{diag}(s)WW' ]_{diag}\\
\end{aligned}
\end{equation}
where $[ \cdot ]_{diag}$ denotes the vector of diagonal elements of matrix $[\cdot]$.
In addition, since we require $s$ to be non-negative, we perform Projected Gradient Descent (PGD)~\cite{calamai1987projected} for this constrained optimization problem.
Specifically, we project $s$ to be non-negative after each gradient updating step:
\begin{equation}
\label{eq_project_s}
\begin{aligned}
\mathrm{Proj}(s_{i}) = max(0, s_{i}), \forall i=1,...,d
\end{aligned}
\end{equation}

In summary, we can alternatively perform the above two steps until it converges or reaches a user-specified maximum iteration number.
For clarity, we summarize this optimization procedure in Alg.~\ref{alg_ours}.



\begin{algorithm}[!tb]
\caption{Semantic based Feature Selection (\myalg)}
\label{alg_ours}
\begin{algorithmic}[1]
\Require
Seen concept instances $X$ and attributes $Y_{s}$;
Parameters $\alpha$ and $\gamma$;
\Ensure Top ranked features for unseen concepts;
    \State Calculate the center of each seen concept $\bar{X}$ ;
    \Repeat
        \State Update $W$ by Eq.~\ref{eq_update_W};
        \State Update $s$ through performing projected gradient descent by Eq.~\ref{eq_update_s};
    \Until{Convergence or a certain iterations;}%
    \State Rank features according to $s$ in a descending order and return the top ranked features.
\end{algorithmic}
\end{algorithm} 
\section{Algorithm Analysis}\label{sect_discuss}

\subsection{Convergence Study}
\begin{proposition} \label{pro_alg}
At each iteration of Alg.~\ref{alg_ours}, the objective value of Eq.~\ref{eq_final_cost_relax} decreases until convergence.
\end{proposition}

\begin{proof}\label{proof_alg}
As shown in Alg.~\ref{alg_ours}, in each iteration, we update $W$ and $s$ in an alternative way.
First of all, when $s$ is fixed, the original optimization problem (i.e., Eq.~\ref{eq_final_cost_relax}) w.r.t. $W$ reduces to the classical least square problem~\cite{duda1995pattern}.
It can be easily verified that Eq.~\ref{eq_update_W} is the optimal solution of this subproblem.


On the other hand, the update of $s$ is known as the projected gradient method which is guaranteed to converge to the minimum with appropriate choice of step size~\cite{luenberger2015linear}.
Therefore, Alg.~\ref{alg_ours} monotonically decreases the objective function value.
In addition, since the objective function has a lower-bound of 0, Alg.~\ref{alg_ours} converges.
\end{proof}

\subsection{Time Complexity}
Lines 3 and 4 in Alg.~\ref{alg_ours} list two main operations of our method, and the time complexity of each operation could be computed as:
\begin{itemize}
  \item Line 3: Updating $W$ involves a matrix inversion and several matrix multiplications, and the total time complexity is
  $O(d^{3} + d^{2}n + dmn)$.
  \item Line 4: Updating $s$ also involves some matrix multiplications, and the time complexity is $O(dmn + d^{2}m + d^{2}n )$.
\end{itemize}
Generally speaking, the feature number and attribute number are much less than the seen instance number, i.e., ${d, m} \ll n $.
Therefore, the total time complexity of each iteration in Alg.~\ref{alg_ours} is $\#iterations * O(n)$.
As our method empirically converges quickly (usually in less than 20 iterations in our experiments), the overall time complexity
of \myalg\ is linear to seen instance number $n$.


Note that if the feature number is large, the matrix inversion in Line 3 could be very time consuming.
According to~\cite{nie2010efficient}, this operation can be efficiently obtained through solving the following linear equation:
\begin{equation}
\label{eq_update_W_linear}
\begin{aligned}
A W {=} [(Xdiag(s))' {+} \alpha (\bar{X}diag(s))']Y_{s}
\end{aligned}
\end{equation}
where $ A {=} (Xdiag(s))'(Xdiag(s)) {+} \alpha (\bar{X}diag(s))'(\bar{X}diag(s)) \\{+} \gamma I_{d}$.
As such, the time complexity of Line 3 would be reduced to $O(md^{2} + d^{2}n + dmn)$.

\subsection{\myalg\ v.s. Traditional Feature Selection Methods}
Traditional feature selection methods including both unsupervised and (semi-)supervised methods all exploit the knowledge of seen concepts rather than unseen concepts~\cite{chandrashekar2014survey}.
Specifically, unsupervised methods generally prefer the features best preserving the intrinsic structure of seen concept data.
(Semi-)supervised methods prefer the features best reflecting the discrimination (i.e., class labels) among different seen concepts.
Consequently, as the data may vary dramatically among totally different concepts, traditional methods may not generalize well to unseen concepts.

Conversely, our method prefers the features best reflecting the knowledge about the attributes of seen concepts.
As it is practicable to categorize unseen objects via inferring attributes, the selected features in our method would also have this ability, which is further verified in our later experiments.

\section{Experiments}\label{sect_expriment}
\begin{table}[!t]
\caption{The statistics of datasets}
\centering 
\begin{tabular}{l|rrr}
\hline
                   &SUN &aPY      &CIFAR10 \\
\hline
\hline
\# seen concepts           &707      &20       &2    \\
\# seen images            &14,140  &12,695     &12,000  \\
\# unseen concepts           &10      &12         &8 \\
\# unseen images            &200   &2,644      &48,000 \\
\# attributes                &102     &64         &50    \\
\# features                  &4,096  &4,096      &4,096 \\
\hline
\end{tabular}
\label{tab_dataset}
\end{table}
\subsection{Experimental Setup}

\mysubsection{Datasets}
The experiments are conducted on three widely used benchmark datasets.
The first dataset is SUN scene attributes database (\textbf{SUN})\footnote{\url{https://cs.brown.edu/~gen/sunattributes.html}}~\cite{patterson2012sun} which contains 14,340 images from 717 different scenes like ``village'' and ``airport''.
For each image, a 102-dimensional binary attribute vector is annotated manually.
We average the images' attributes to obtain the attributes for each concept, and follow the seen/unseen split setting as~\cite{jayaraman2014zero}.

The second dataset is aPascal/aYahoo objects dataset (\textbf{aPY})\footnote{\url{http://vision.cs.uiuc.edu/attributes/}}~\cite{farhadi2009describing} which contains two subsets.
The aPascal subset comes from PASCAL VOC2008 dataset, and contains 20 different categories, such as ``people'' and ``dog''.
The aYahoo subset is collected by Yahoo image search engine, and has 12 similar but different categories compared to aPascal, such as ``monkey'' and ``zebra''.
This dataset has a standard split setting, i.e., aPascal and aYahoo serve as the seen part and unseen part respectively.
Each image is annotated by a 64-dimensional binary attribute vector.
We average the attribute vectors of images in the same category to get the class attributes.

The third dataset is \textbf{CIFAR10}\footnote{\url{https://www.cs.toronto.edu/~kriz/cifar.html}}~\cite{krizhevsky2009learning} which consists of 10 classes of objects with 6,000 images per class.
Since this dataset does not have a standard seen/unseen split setting, we randomly adopt two classes as seen.
Thus, we have $C_{10}^{2}$ different seen/unseen splits.
The purpose of this setting is to test the extreme case where most of the concepts are unseen.
In addition, as this dataset does not have any attribute annotations, we adopt the 50-dimensional word vector provided by~\cite{huang2012improving} as attributes for each class.
Concretely, these embedding vectors are automatically generated from a Wikipedia corpus~\cite{shaoul2010westbury} with a total of about 2 million articles and 990 million tokens.


For all these three datasets, we use the widely used 4,096-dimensional deep features provided by~\cite{guo2016semi}.
The detailed statistics of these three benchmarks are listed in Table~\ref{tab_dataset}.


%

\subsubsection{Comparison methods}
We compare the proposed method \myalg\ with both unsupervised and supervised feature selection methods:
\begin{enumerate}
  \item Random method selects features randomly.
  \item MCFS~\cite{cai2010unsupervised} is an unsupervised method selecting features by using spectral regression with $\ell_{1}$-norm regularization.
  \item FSASL~\cite{du2015unsupervised} is an unsupervised method performing structure learning and feature selection simultaneously.
  \item LASSO~\cite{tibshirani1996regression} is the classical supervised feature selection method.
  \item L20ALM~\cite{cai2013exact} is a supervised method evaluating features under an $\ell_{21}$-norm loss function with $\ell_{20}$-norm constraint.
  \item WkNN~\cite{bugata2019weighted} is a newest supervised feature selection method based on k-nearest neighbors algorithm
\end{enumerate}
In addition, to validate the effectiveness of the proposed center-characteristic loss, we test a variant of our method (denoted as \emph{\myalgv}) by setting $\alpha {=} 0$ in Eq.~\ref{eq_final_cost_relax} to eliminate the effect of this loss term.
Another point to be noted is that we apply all methods on seen concepts and then evaluate the quality of selected features on unseen concepts.
In other words, in the training phase, the unseen concepts are invisible to all methods.

\begin{figure*}[!t]
\centering
\subfigure[SUN]{
    \includegraphics[width=0.345\textwidth]{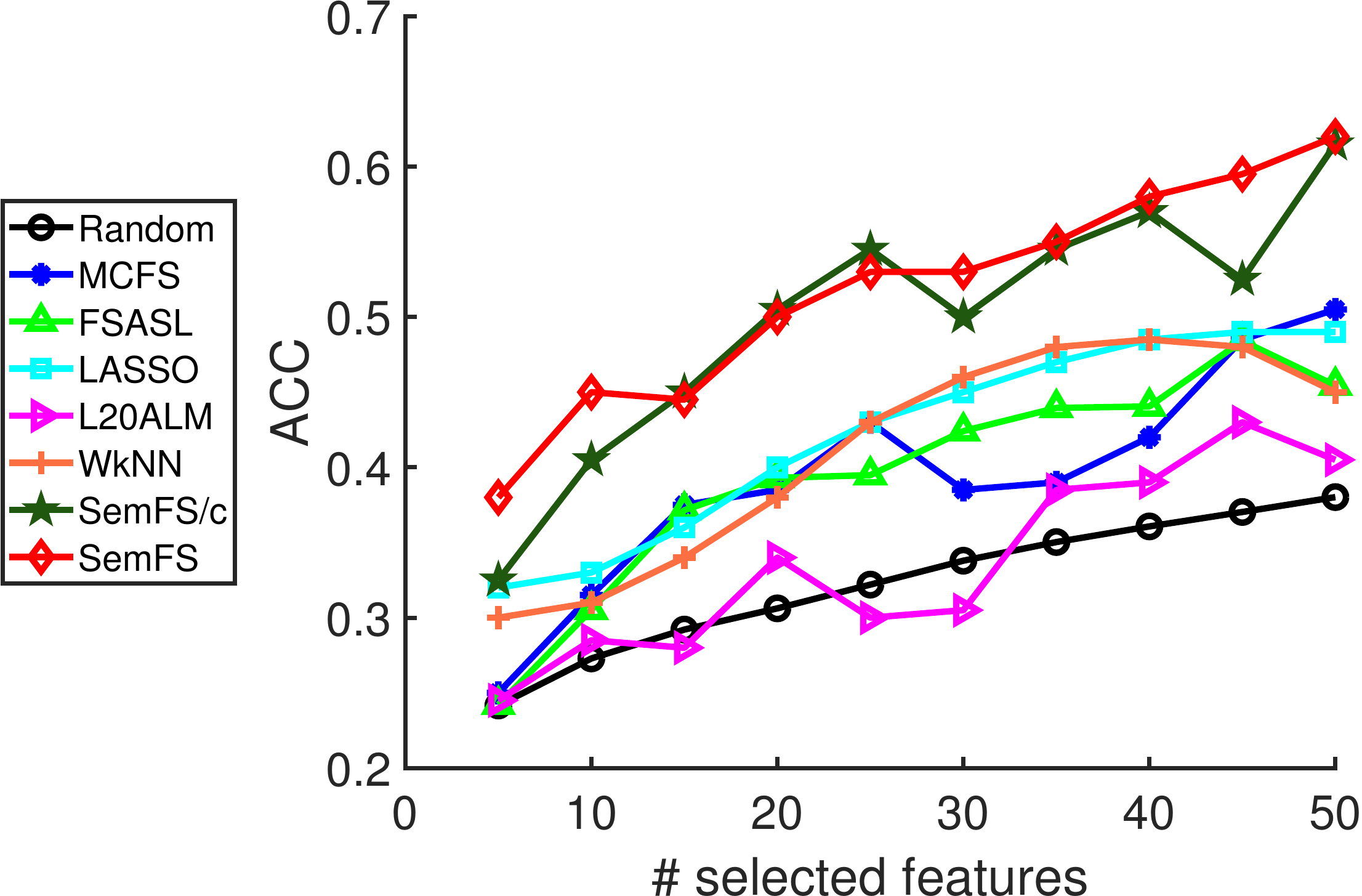}
}
\subfigure[aPY]{
    \includegraphics[width=0.285\textwidth]{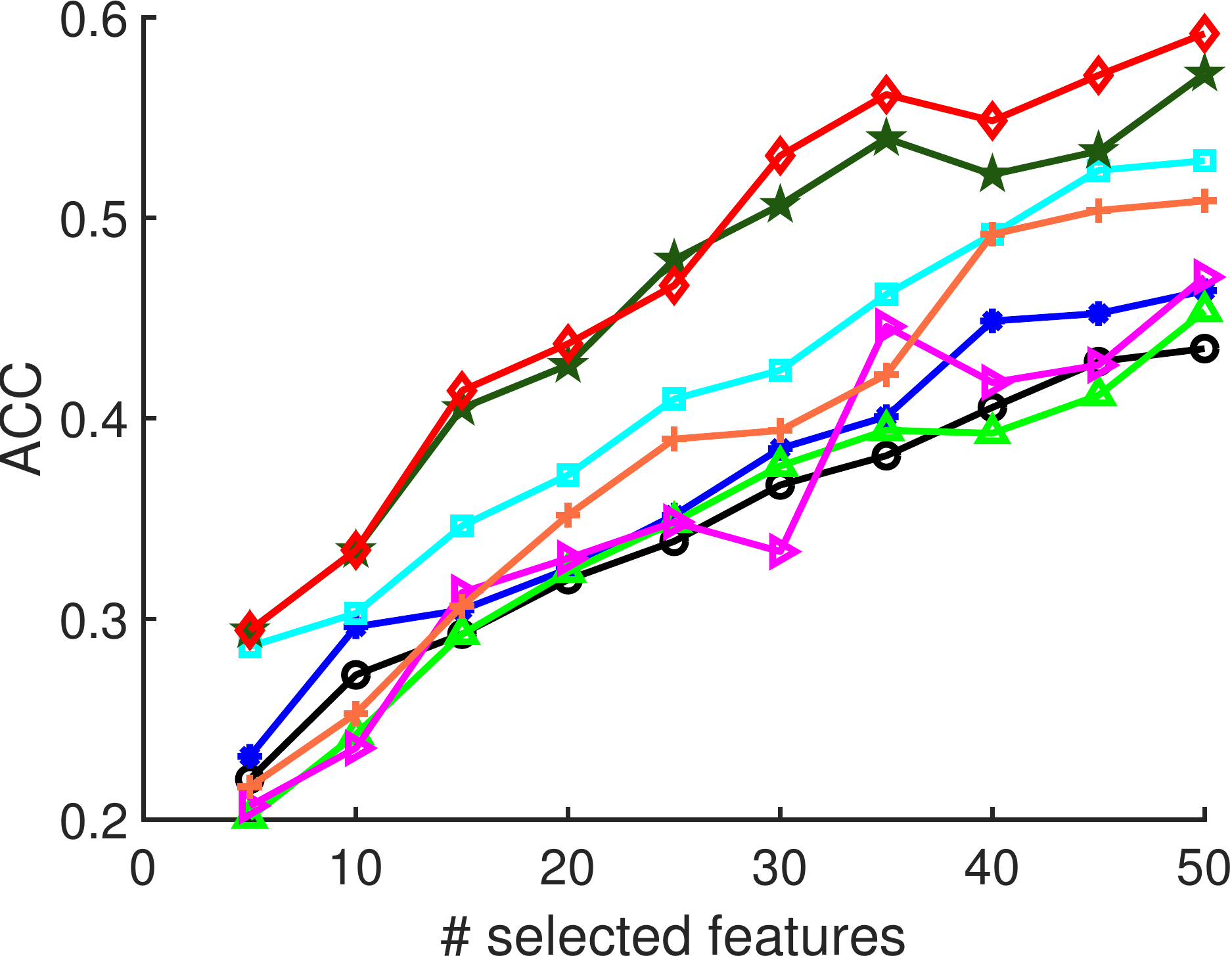}
}
\subfigure[CIFAR10]{
    \label{fig:ACC_CIFAR10}
    \includegraphics[width=0.285\textwidth]{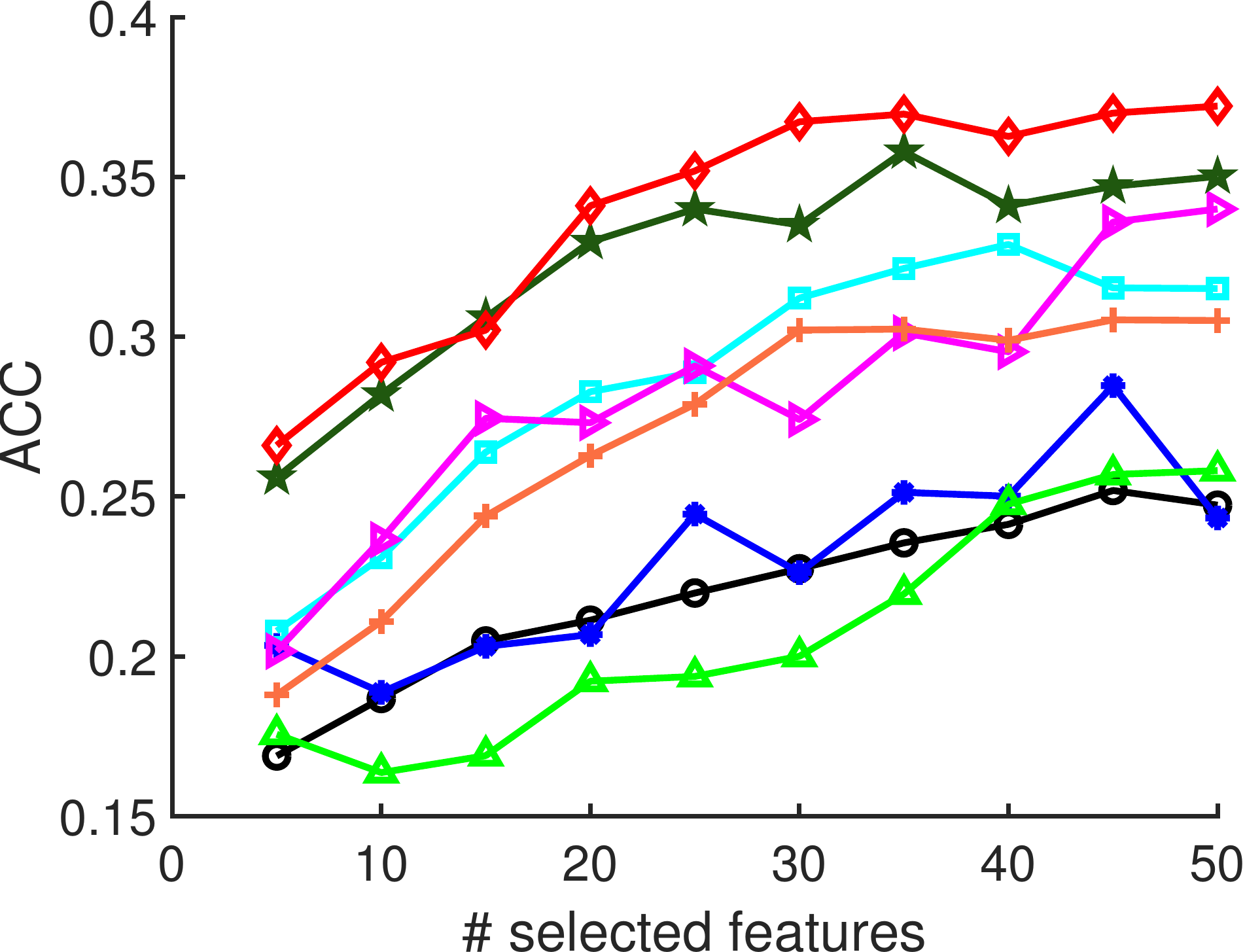}
}
\caption{Clustering results (ACC).}
\label{fig_ACC}
\end{figure*}

\begin{figure*}[t]
\centering
\subfigure[SUN]{
    \includegraphics[width=0.345\textwidth]{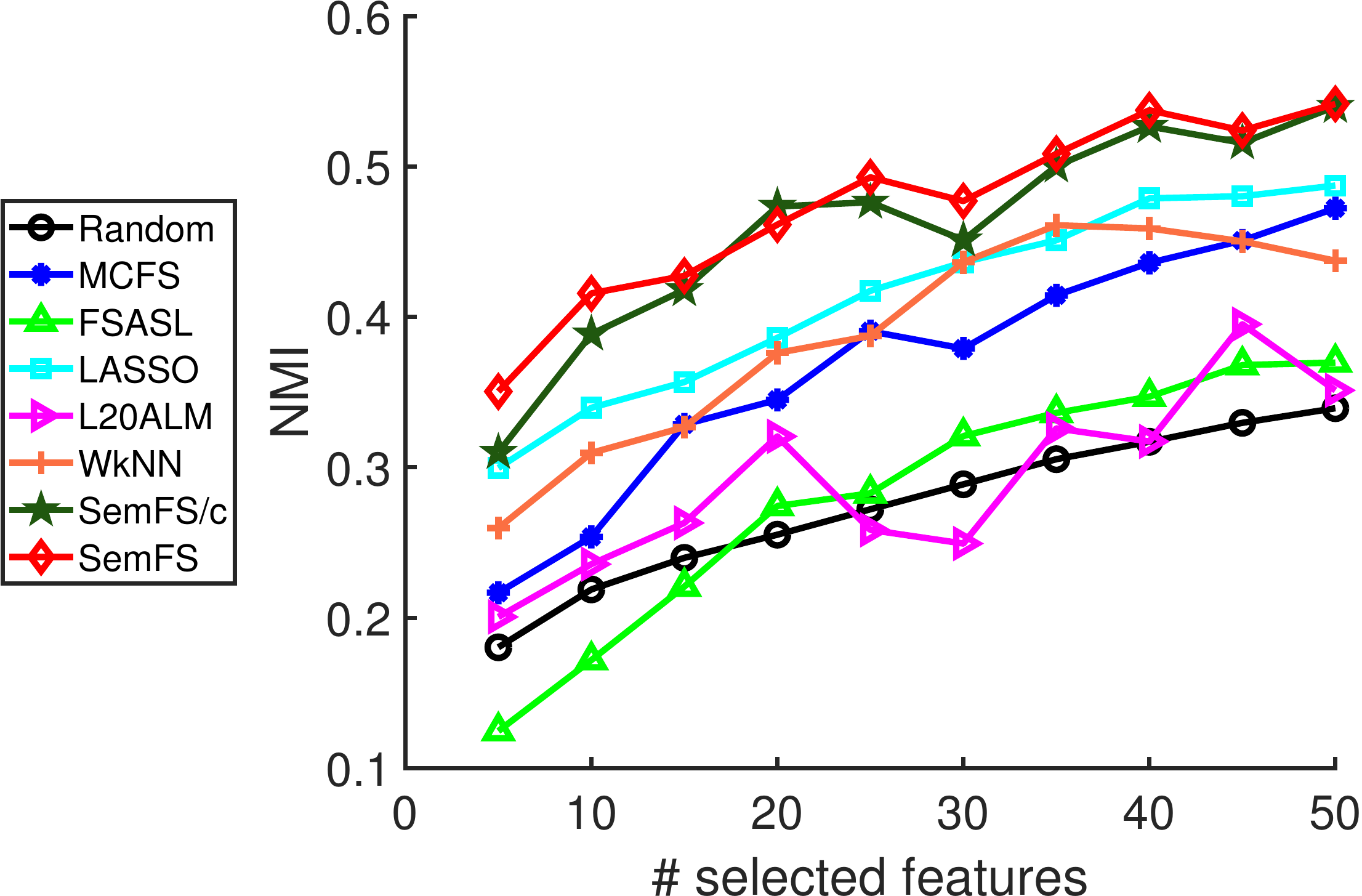}
}
\subfigure[aPY]{
    \includegraphics[width=0.285\textwidth]{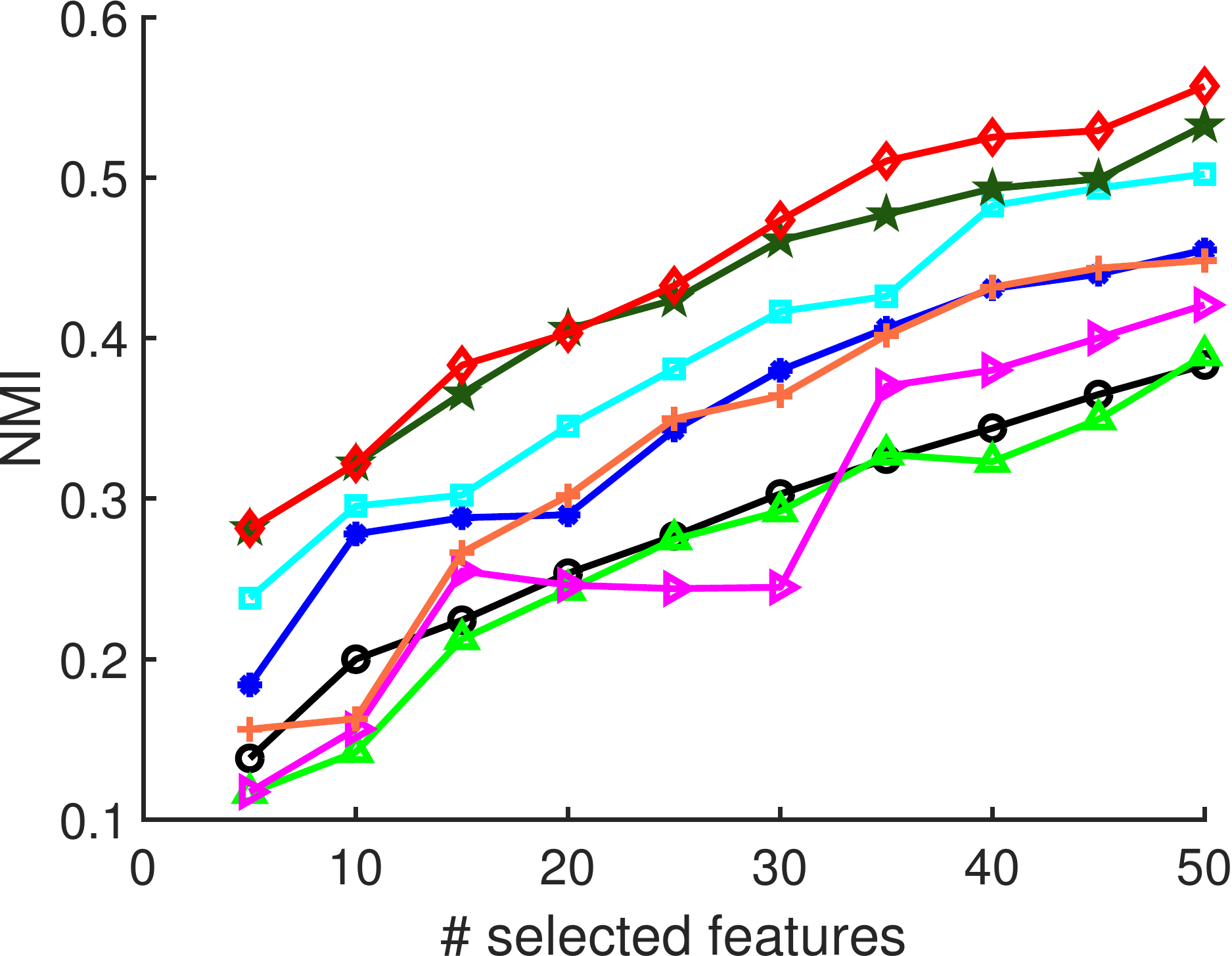}
}
\subfigure[CIFAR10]{
    \label{fig:NMI_CIFAR10}
    \includegraphics[width=0.285\textwidth]{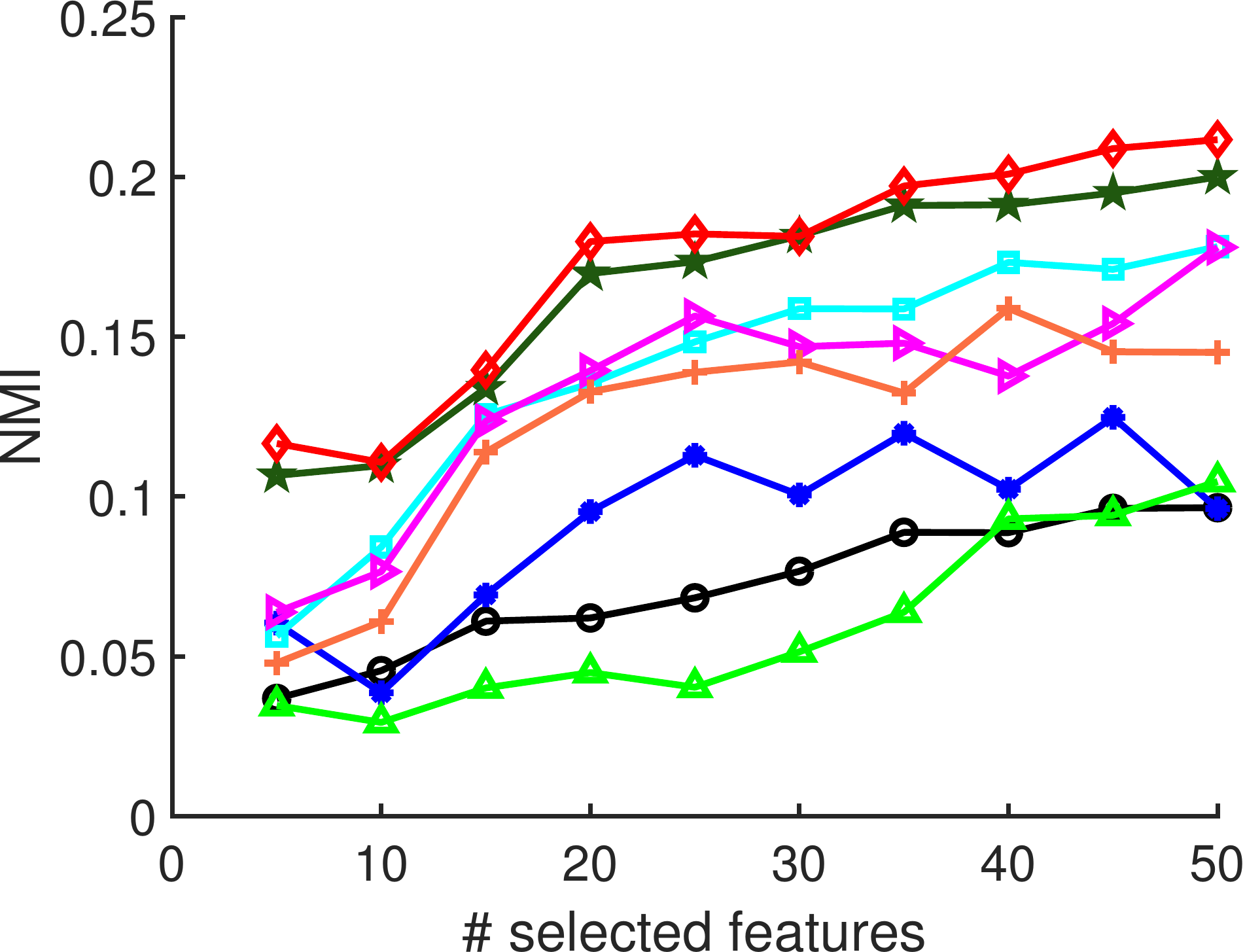}
}
\caption{Clustering results (NMI).}
\label{fig_NMI}
\end{figure*}
\subsubsection{Experimental Setting}
Since determining the optimal number of selected features is still an open problem~\cite{tang2012unsupervised}, we follow~\cite{du2015unsupervised} to vary the selected features number from 5 to 50 in intervals of 5.
As the unseen concepts have no labeled data, we evaluate the selected features on the clustering task, so as to follow the current practice of evaluating unsupervised feature selection~\cite{cai2010unsupervised}.
Specifically, we employ the classical \emph{K}-means\footnote{\url{www.cad.zju.edu.cn/home/dengcai/Data/code/litekmeans.m}} clustering method and adopt two widely used clustering metrics, i.e., Accuracy (ACC) and Normalized Mutual Information (NMI)\footnote{\url{www.cad.zju.edu.cn/home/dengcai/Data/Clustering.html}}~\cite{duda2012pattern}.
Since \emph{K}-means is sensitive to initialization, we repeat the clustering 20 times with random initializations and report the average performance.
To fully show the limitations of these compared methods, we tune their parameters by a grid-search strategy from $\{10^{-2},10^{-1},10^{0},10^{1},10^{2}\}$ and report the best results.
In contrast, to show the effectiveness of our method, we simply fix our parameters $\alpha{=}1$ and $\gamma{=}0.1$ throughout the experiment.

\subsection{Clustering with Selected Features}
Figure~\ref{fig_ACC} and Figure~\ref{fig_NMI} show the clustering performance in terms of ACC and NMI, respectively.
There are several important observations to be made.
\begin{itemize}
  \item Our methods (both \myalg\ and \myalgv) always outperform the others significantly, in terms of both ACC and NMI, on all datasets.
  For example, with 20 selected features, our two methods outperform the best baseline by 20–-40\% relatively in terms of ACC.
The underlying principle is that our methods successfully deduce the knowledge of unseen concepts from seen concepts by introducing attributes.
  \item An interesting observation is that the unsupervised method MCFS is comparable to or even better than those supervised baselines.
  This indicates that the supervised methods may be misled by the seen concept labels which provide little information about unseen concepts.
  \item On the CIFAR10 dataset where most concepts are unseen and attributes are automatically generated, our methods could still identify high quality features (Figs.~\ref{fig:ACC_CIFAR10} and~\ref{fig:NMI_CIFAR10}).
This confirms the flexibility of our methods as well as the advantage of attributes.
  \item \myalg\ consistently obtains more stable and better performance than \myalgv.
This phenomenon indicates that the proposed center-characteristic loss can help to select reliable discriminative features for unseen concepts with limited labeled seen instances.
\end{itemize}

\begin{figure}[!t]
\centering
\subfigure{
    \includegraphics[width=0.45\textwidth]{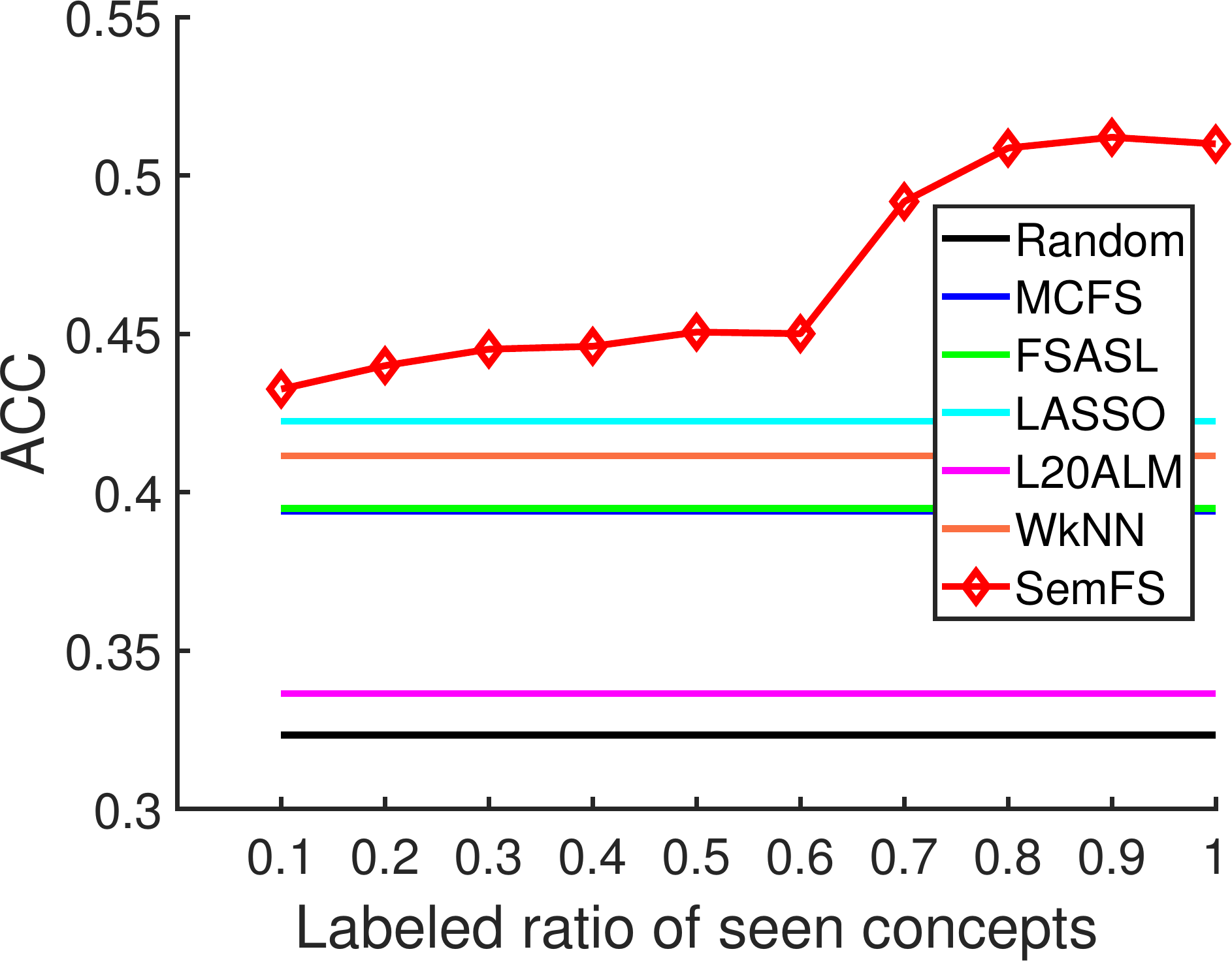}
}
\subfigure{
    \includegraphics[width=0.45\textwidth]{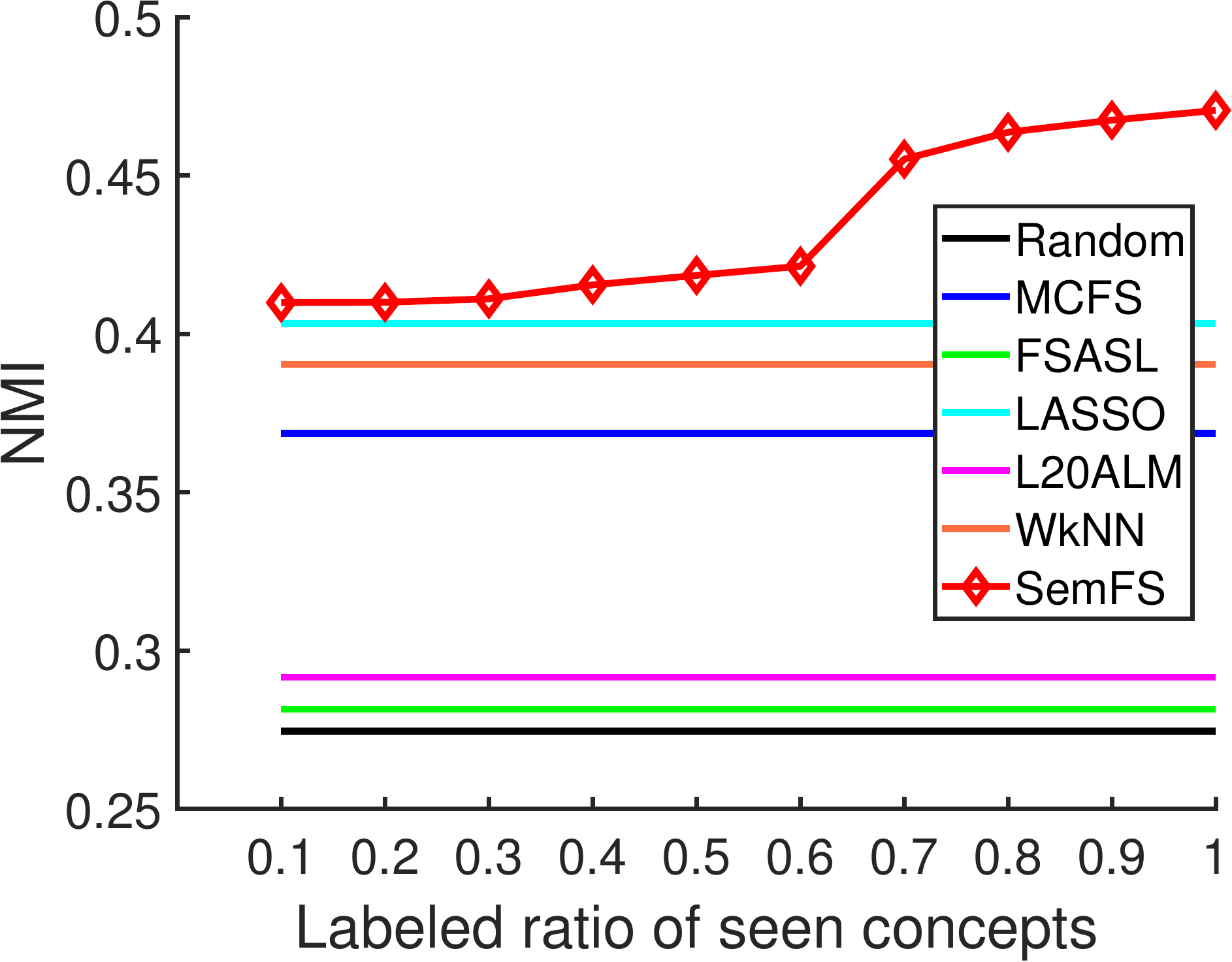}
}
\caption{Average clustering performance w.r.t. different labeled ratios of seen concepts on SUN.}
\label{fig_effect_ratio}
\end{figure}

\subsection{Effect of Labeled Seen Concept Ratio}
In this part, we evaluate the performance of the proposed method \myalg\ w.r.t. the labeled ratio in all seen concepts.
Specifically, in our method, we vary the labeled ratio of each seen concept from 0.1 to 1.
Figure~\ref{fig_effect_ratio} shows the average clustering performance on SUN (due to space limitations, we do not report the results on the other datasets, because we have similar observations on them).
It can be observed that even with only 10\% of seen concepts, \myalg\ could still benefit the success of cross-concept knowledge generalization.
This verifies the superiority of attributes for the ZSFS problem.
On the other hand, we find that as labeled ratio of seen concepts increases, our method shows an increasing advantage over the previous studies.
We analyse that with more seen concepts, \myalg\ would deduce more reliable supervision for the unseen concepts, thereby simultaneously improving the quality of selected features.


\subsection{Advantage of Attributes}
To further show the advantage of attributes, we use three methods: \myalg, \myalgv\ and the compared classical feature selection method LASSO.
For clarity, we use $\mathcal{M}_{o}$ and $\mathcal{M}_{s}$ to denote the method $\mathcal{M}$ using the original class labels (i.e., $Y$) and semantic attributes (i.e., $Y_{s}$), respectively.
Figure~\ref{fig_semantic_compare} shows the average clustering results in terms of ACC.
It can be clearly observed that: compared to the original class labels, attributes could significantly improve the generalization ability of these feature selection methods.
The underlying reason is that attributes contain discriminative information about unseen concepts.
In contrast, the original class labels neglect this kind of knowledge, thereby leading to poor performance.


\begin{figure}[!t]
\centering
    \includegraphics[width=0.6\textwidth]{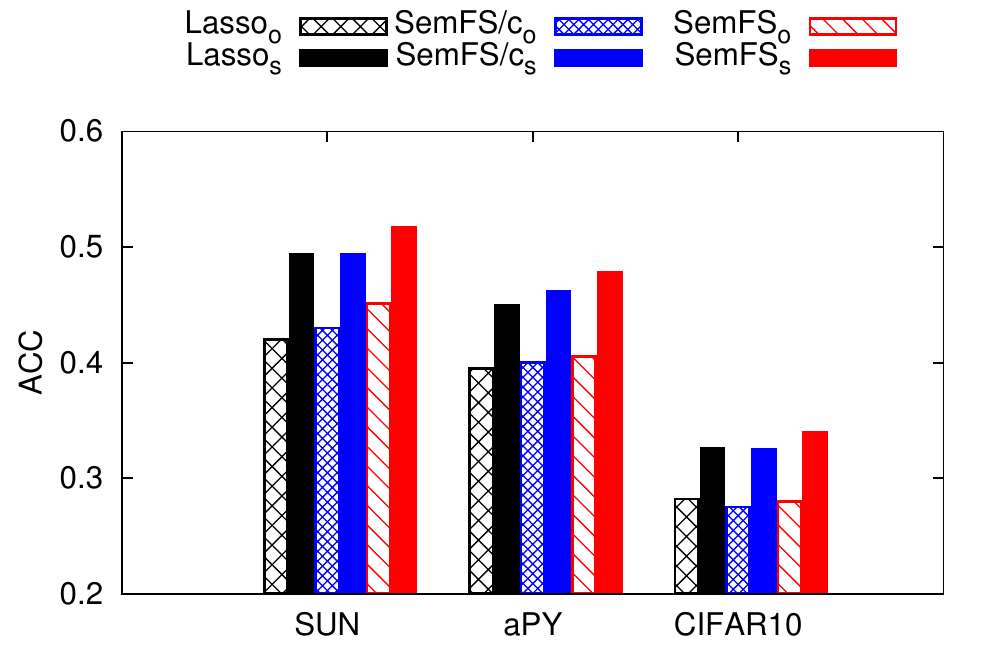}
\caption{Average ACC: class labels vs. attributes.}
\label{fig_semantic_compare}
\end{figure}

\begin{figure}[t]
\centering
\subfigure{
    \includegraphics[width=0.45\textwidth]{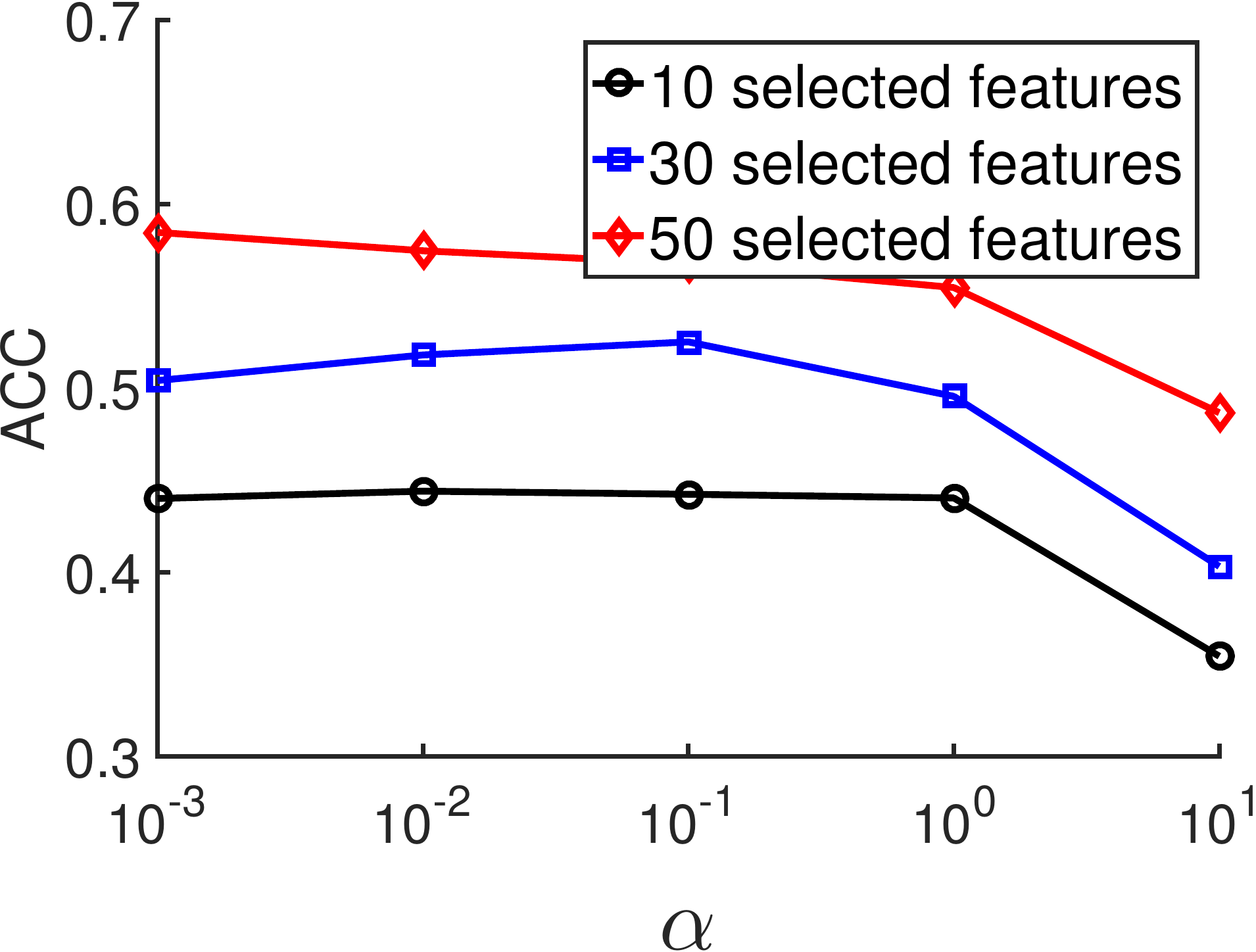}
}
\subfigure{
    \includegraphics[width=0.45\textwidth]{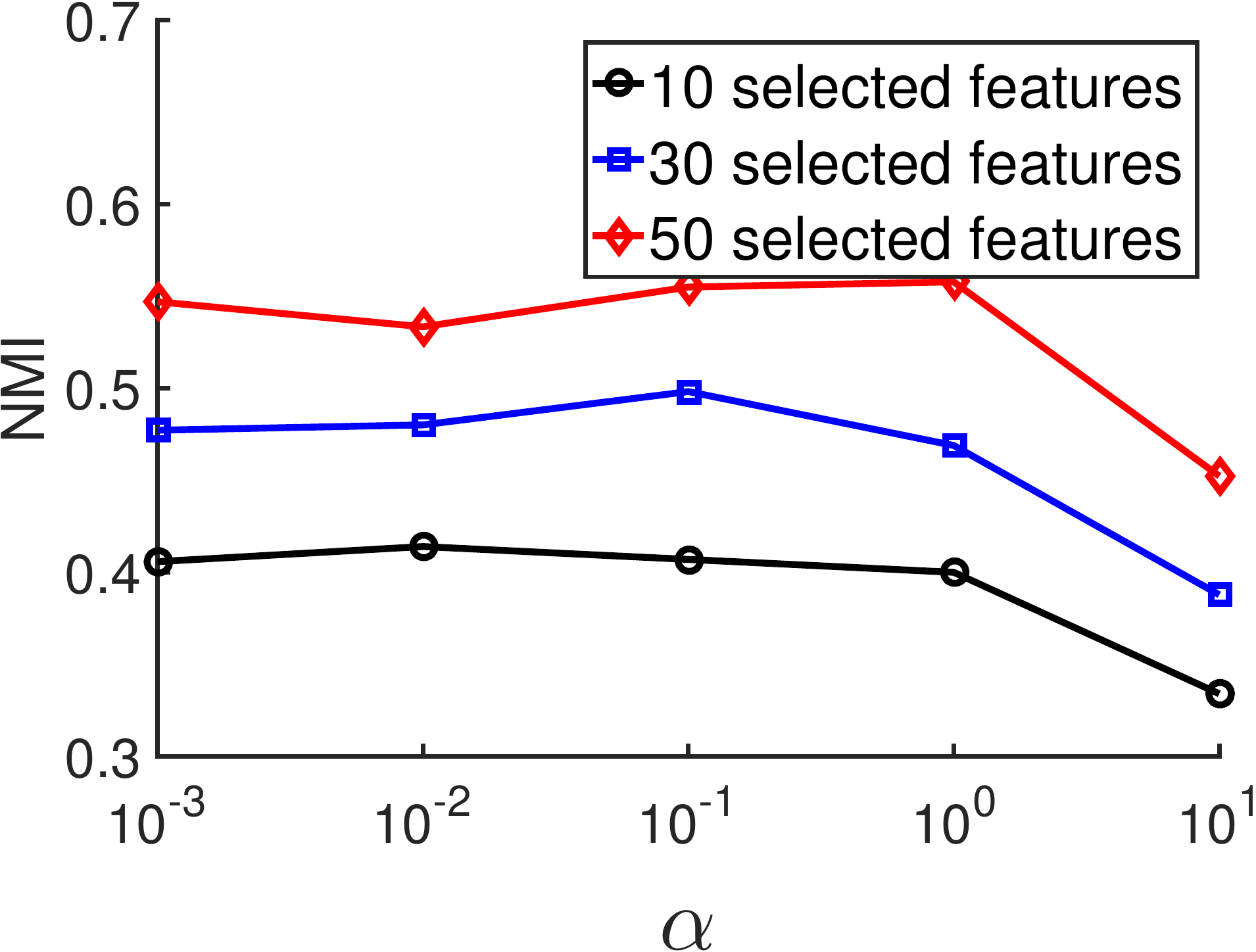}
}
\caption{The effect of parameter $\alpha$ on SUN.}
\label{fig_para}
\end{figure}

\subsection{Parameter Sensitivity}
In the proposed method \myalg, there is an important parameter $\alpha$ which controls the penalty of center loss.
Figure~\ref{fig_para} shows the clustering performance w.r.t. this parameter on SUN (with the fixed regularization parameter $\lambda{=}1$).
It can be observed that our method is not very sensitive to $\alpha$, and the results remain almost the same when the parameters vary among $\{10^{-3}, ..., 10^{0} \}$ in our experiments.

\section{Conclusion} \label{section_conclusion}
This paper investigates the problem of Zero-Shot Feature Selection, i.e., building a feature selection model that generalizes well to unseen concepts with limited training data of seen concepts.
We propose a novel feature selection method named \myalg.
The basic idea of our method is to guide feature selection by attributes which can be seen as a bridge for transferring supervised knowledge from seen concepts to unseen concepts.
In addition, we propose the center-characteristic loss to enhance the quality of selected features.
Finally, we formulate these two components into a joint learning model and give an efficient solution.
Extensive experiments conducted on several real-world datasets demonstrate the effectiveness of our method.
In the future, we plan to extend our method to use the additional unlabeled data of seen concepts.



\section*{Acknowledgment}

\noindent
This work is supported in part by the Fundamental Research Funds for the Central Universities (Grant No.~FRF-TP-18-016A1) and National Natural Science Foundation of China (No.~61872207).
\vskip 2mm





\bibliographystyle{elsarticle-harv}
\bibliography{simple}

%
%
%
%

%
%
%

\end{document}